\documentclass[review]{elsarticle}

\usepackage{microtype}
\usepackage{graphicx}
\usepackage{subfigure}
\usepackage{booktabs}
\usepackage{hyperref}
\usepackage{multirow}
\usepackage{amsmath}
\usepackage{amssymb}
\usepackage{mathtools}
\usepackage{amsthm}
\usepackage[capitalize,noabbrev]{cleveref}

\usepackage{mathtools}
\usepackage{amsthm}
\usepackage{algorithm}
\usepackage{algorithmic}
\usepackage{amsfonts}
\usepackage{amsbsy}
\usepackage{xspace}
\usepackage{textcomp}
\usepackage{multirow}

\theoremstyle{plain}
\newtheorem{theorem}{Theorem}[section]

\theoremstyle{definition}

\theoremstyle{remark}

\newcommand\scalemath[2]{\scalebox{#1}{\mbox{\ensuremath{\displaystyle #2}}}}

\makeatletter
\DeclareRobustCommand\onedot{\futurelet\@let@token\@onedot}
\def\@onedot{\ifx\@let@token.\else.\null\fi\xspace}

\def\eg{\emph{e.g}\onedot,\;} 

\def\ie{\emph{i.e}\onedot,\;} 

\def\etc{\emph{etc}\onedot}

\def\etal{\emph{et al}\onedot}

\journal{Journal of \LaTeX\ Templates}

\begin{document}

\begin{frontmatter}

\title{General Rotation Invariance Learning for Point Clouds via Weight-Feature Alignment}

\author[zjut]{Liang Xie}
\author[jd]{Yibo Yang}
\author[zjusoft]{Wenxiao Wang}
\author[zjusoft]{Binbin Lin}
\author[zju]{Deng Cai}
\author[zju,fabu]{Xiaofei He}
\author[zjut]{Ronghua Liang}

\address[zjut]{Colleague of Computer Science and Technology, Zhejiang University of Technology, Hangzhou, China}
\address[zju]{State Key Lab of CAD \& CG, Zhejiang University, Hangzhou, China}
\address[zjusoft]{School of Software Technology, Zhejiang University, Ningbo, China}
\address[jd]{JD Explore Academy, Beijing, China}
\address[fabu]{Fabu Inc., Hangzhou, China}

\begin{abstract}
Compared to 2D images, 3D point clouds are much more sensitive to rotations. We expect the point features describing certain patterns to keep invariant to the rotation transformation. 
There are many recent SOTA works dedicated to rotation-invariant learning for 3D point clouds. 
However, current rotation-invariant methods lack generalizability on the point clouds in the open scenes due to the reliance on the global distribution, \ie the global scene and backgrounds. 
Considering that the output activation is a function of the pattern and its orientation, we need to eliminate the effect of the orientation. 
In this paper, inspired by the idea that the network weights can be considered a set of points distributed in the same 3D space as the input points, we propose Weight-Feature Alignment (WFA) to construct a local Invariant Reference Frame (IRF) via aligning the features with the principal axes of the network weights. 
Our WFA algorithm provides a general solution for the point clouds of all scenes. 
WFA ensures the model achieves the target that the response activity is a necessary and sufficient condition of the pattern matching degree. 
Practically, we perform experiments on the point clouds of both single objects and open large-range scenes. 
The results suggest that our method almost bridges the gap between rotation invariance learning and normal methods.
\end{abstract}


\end{frontmatter}


\section{Introduction}

\begin{figure}[t]
    \includegraphics[width=1.00\columnwidth]{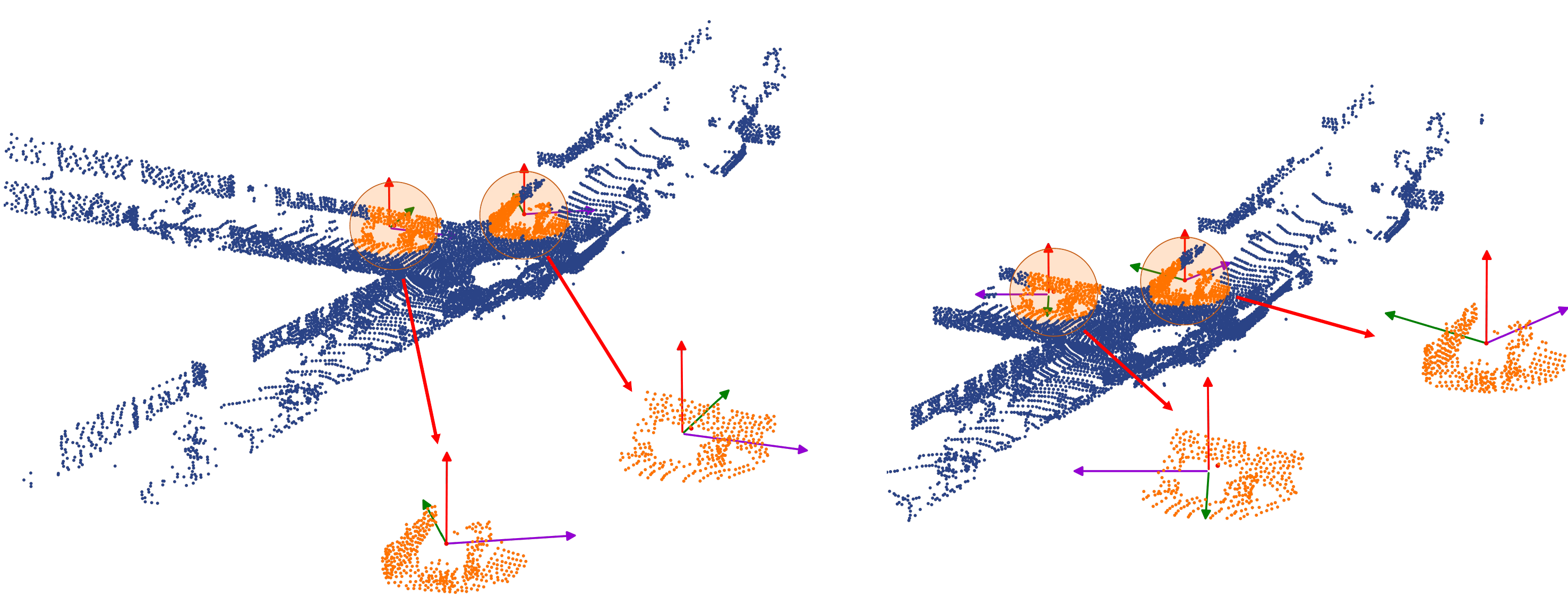}
    \centering
    \caption{\textbf{The Drawback of Relying on the Global Scene.} For most of the previous IRF-based methods, construct IRFs relied on the global distribution of the point clouds, such as the barycenter, \etc. When applying these methods to open scenes, IRFs will vary dramatically as the scene changes. This drawback severely limits the generalizability in the open scenes. }
    \label{fig:drawback_reliance_on_global_distribution}
\end{figure}

In 2D RGB-images, the object semantic information is mainly carried by the pixel colors and patch textures. The features extracted by a 2D convolution kernel are not affected by the pixel coordinates. However, in the point clouds captured by LiDAR, it becomes completely different. The point features are their Cartesian coordinates. Although some LiDAR might obtain an extra intensity value to describe the reflection ratio of the surfaces, the main discriminative features are still the 3D shape patterns.
Moreover, the 3D point positions are continuous in $\mathbb{R}^{3}$ compared to the image pixels' discrete positions in $\mathbb{Z}^{2}$, which means that the rotation in 3D space has more freedom degree than 2D rotations. Therefore, the point features will change dramatically when their coordinates change due to rotation. However, the 3D patterns described by the point clouds keep invariant. Naturally, extracting invariant features to describe invariant patterns becomes an important topic worth researching.

Some previous works \cite{cohen2016group,cohen2018spherical,weiler20183d,cohen2016steerable,shen20203d,worrall2018cubenet,zhao2020quaternion,li2021leveraging,rao2019spherical,poulenard2019effective,liu2018deep} design rotation-equivariant operators. Inspired by the convolution is equivariant to the translation, they perform the equivariant operations via projecting the 3D points in the Cartesian reference frame to the spherical reference frame because the rotation becomes the translation transform in the spherical reference frame. Equivariance can be considered the generalized invariance. However, to eliminate the effect of rotation, the network must employ a symmetric function, \eg pooling operation, to obtain the final invariant features to achieve the supervised training for the classification task. It limits the performance of equivariance learning. Therefore, there is still a performance gap between equivariance learning and invariance learning. Note that many tasks, such as segmentation, object detection,\etc, contain the classifying operation. Therefore, this drawback limits the generalization of these methods.

Some other studies explore the rotation-invariant features, which can be divided into two groups: RIF-based and IRF-based methods. RIF-based methods \cite{yu2020deep,zhang2019rotation,zhang2022riconv++,deng2018ppf,sun2019srinet,li2021rotation,xiao2021triangle,zhao2022rotation,chen2019clusternet} attempt to build RIF (Rotation-Invariant Features) via the point relative relationships information. However, their designed features are distributed in a non-orthogonal feature space, which brings mathematical redundancy. Moreover, their designs are theoretically interpretable due to that they cannot ensure an optimal feature space. IRF-based methods \cite{spezialetti2020learning,fang2020rotpredictor,xiao2020endowing,li2021closer,kim2020rotation,zhao2022rotation,zhang2020learning} dedicate to build a local or global rotation-invariant reference frame via PCA (Principle Component Analysis) or other algorithms. However, most of their constructed IRFs highly rely on the global point distribution, \eg IRFs in \cite{zhang2020learning} the global barycenter and IRFs in \cite{zhao2022rotation} the global principal components. As illustrated in Figure \ref{fig:drawback_reliance_on_global_distribution}, although the features are invariant to the rotation transformation, they are not invariant to the global scene, \ie the features are affected by not only the local patterns but also the surroundings and backgrounds. It causes that they can be applied to the point clouds of small scenes, \eg the points clouds describing a single object or scanned from small rooms, and is not applicable for open outdoor scenes. This drawback severely limits their practical applicability. Although the PCA only based on local region \cite{kim2020rotation} can solve this problem, they still face the pose ambiguity problem caused by the uncertainty of the orders and orientations of the principal axes. Some works \cite{xiao2020endowing,li2021closer,yu2020deep,spezialetti2020learning,fang2020rotpredictor} try to solve this issue via learning a fixed orientation or combining all possibilities. However, the learned orientations or combining weights are not interpretable. For similar patterns, the network might do completely different estimations.

To overcome the above drawbacks, we rethink the fundamental reason for the sensitivity to rotation. The feature extraction in the local region can be considered through the following procedure:

\begin{equation}
	\label{equ:activation_unexpected}
	\mathrm{activation} = f\left(\mathrm{pattern}, \; \mathrm{orientation}, \; \mathrm{background}\right)
\end{equation}

\noindent It means that the output activation value is a function of the local pattern and its orientation. The parameter \textit{background} indicates that some rotation-invariant methods are affected by the background to achieve rotation invariance, which brings another unexpected noisy variance. Ignoring the \textit{background} parameter, the kernel function $f$ will give an active response if and only if the pattern and orientation are matched to the detected target of the kernel. In this case, as Figure (\ref{fig:pattern_orientation_alignment}) shows, an inactive response does not mean the pattern is not expected because the inactivity might be caused by the misaligned orientations between the kernel and input. Therefore, we expect the activation is a function only based on a local pattern, \ie

\begin{equation}
	\label{equ:activation_expected}
	\mathrm{activation} = f\left(\mathrm{pattern}\right)
\end{equation}

\noindent So that the response activity can be thought of as a necessary and sufficient condition of the pattern matching degree.

In this paper, inspired by the idea that the weights of the network can be considered a set of points and the feature extraction can be considered pattern matching, we propose to align the points and weights. Our Weight-Feature Alignment (WFA) algorithm transforms the points to the reference frame decided by the principal axes of the network weights. Different from the previous rotation-invariant works limited in the applicable scenes due to the high reliance on the global distribution of the point clouds, our WFA algorithm can be considered a general operator applicable in the point clouds of arbitrary scenes, \eg the open and large-range outdoor scenes. The experiment results show that our method almost bridges the gap between rotation invariance learning and normal methods\footnote{methods that are designed universally and do not consider rotation invariance, \eg PointNet++, evaluating on the unrotated data}.

The contributions of this paper can be summarized as:

\begin{itemize}
	\item We propose WFA (Weight-Feature Alignment) to align the point features and network weights. Our approach provides a general solution for robust invariant feature learning and can be applied on the points in all scenes.
	\item We perform a theoretical analysis to provide a mathematical explanation for the effectiveness, which makes our method more interpretable. 
	\item We perform experiments on the point clouds of variance scenes, including the single-object and open scenes to demonstrate the generalness. The results suggest that our method almost bridges the performance cap between rotation invariance learning and normal methods.
\end{itemize}

\section{Related Works}\label{sec:related_works}

\subsection{Mainstream 3D Vision Architecture}

In the field of 3D vision, the mainstream architectures \cite{qi2016pointnet,qi2017pointnetplusplus,wang2019dynamic,wu2018pointconv,thomas2019KPConv,li2018so,wang2018deep,li2018pointcnn,atzmon2018point} are similar to the CNNs (Convolution Neural Networks) widely applied in 2D computer vision. They extract the point features via aggregating the features of the points in the neighboring region. The feature aggregation can be interpreted as the practical implementation of the mathematical continuous convolution kernel via MLP, theoretically supported by the Universal Approximation Theorem \cite{gybenko1989approximation,csaji2001approximation}. The range of the neighboring region corresponds to the receptive field of the 2D CNN's convolutional kernel. Then, these frameworks extract features from different fields, from local to global, by employing similar hierarchical architectures to CNN. The recent rotation invariance learning of point clouds adopts a similar idea. The network extracts local rotation-invariant features and hierarchically aggregates them layer by layer. Therefore, the current mainstream idea of rotation invariance learning is to discover a local invariant operator.

\subsection{Rotation Equivariance Learning}

Equivariance means that the features extraction function and the transformation are exchangeable. Rotation-equivariant works \cite{cohen2016group,esteves2018learning,cohen2018spherical,weiler20183d,cohen2016steerable,shen20203d,worrall2018cubenet,zhao2020quaternion,li2021leveraging,rao2019spherical,poulenard2019effective,liu2018deep} design rotation equivariant operators to achieve rotation robustness. Their main idea is to project the 3D points in the Cartesian reference frame to the spherical coordinate system where the rotation transformation becomes translation. Because the convolution is equivariant to the translation transform, the feature extraction on the spherical reference frame is also equivariant to rotation. Because many machine learning tasks contain the classifying operation, the network must employ symmetric functions, such as the commonly used pooling operation, to eliminate the effect of rotation in the final classification results. However, it severely limits the performance, which causes the current performance gap between Equivariance learning and invariance learning.

\subsection{Rotation Invariance Learning}

The rotation invariance learning can be divided into two groups. One is to design RIF (Rotation-Invariant Features), and the other is to design IRF (Invariant Reference Frame).

\subsubsection{RIF-based Methods.} RIF-based methods \cite{yu2020deep,zhang2019rotation,zhang2022riconv++,deng2018ppf,sun2019srinet,li2021rotation,xiao2021triangle,zhao2022rotation,chen2019clusternet} build the low-level RIF to achieve rotation invariance. They design RIF via encoding the relationships of the neighboring points, such as the relative distances and angles between the query point and its neighboring points. However, their high-dimensional representations are distributed in a non-orthogonal space, which brings mathematical redundancy and theoretical uninterpretability.

\subsubsection{IRF-based Methods.} IRF-based Methods represent the points in a constructed IRF. Some works \cite{xiao2020endowing,li2021closer,zhao2022rotation} build the global IRFs based on the PCA (Principle Component Analysis). However, global-IRF-based methods are not applicable for the point clouds of open scenes due to that the invariance will be destroyed by the effect of the global scene, \ie the features are affected by not only the local patterns but also the surroundings and backgrounds. Meanwhile, some works build local IRFs for each query point but still rely on the global point distribution, \eg IRF in \cite{zhang2020learning} relies on the global barycenter, and IRFs in \cite{zhao2022rotation} rely on the global principal components. These methods can only be applied to the small-scene point clouds, \eg the single-object points clouds, and is not applicable for open scenes, \eg the points clouds captured from the outdoor street scenes. This drawback severely limits their practical applicability. Although \cite{kim2020rotation} performs PCA to build local IRFs without the reliance on the points out of the local region, it still faces the IRF ambiguity problem caused by the order and orientation uncertainty of the principal axes. Some works attempt to solve this problem via learning a fixed orientation \cite{yu2020deep,spezialetti2020learning,fang2020rotpredictor} or combining all possibilities \cite{li2021closer,xiao2020endowing} through an attention-like mechanism. However, the learned orientations or combining weights are not interpretable. For similar patterns, the network might do completely different estimations.

\begin{figure*}[t]
    \includegraphics[width=0.86\textwidth]{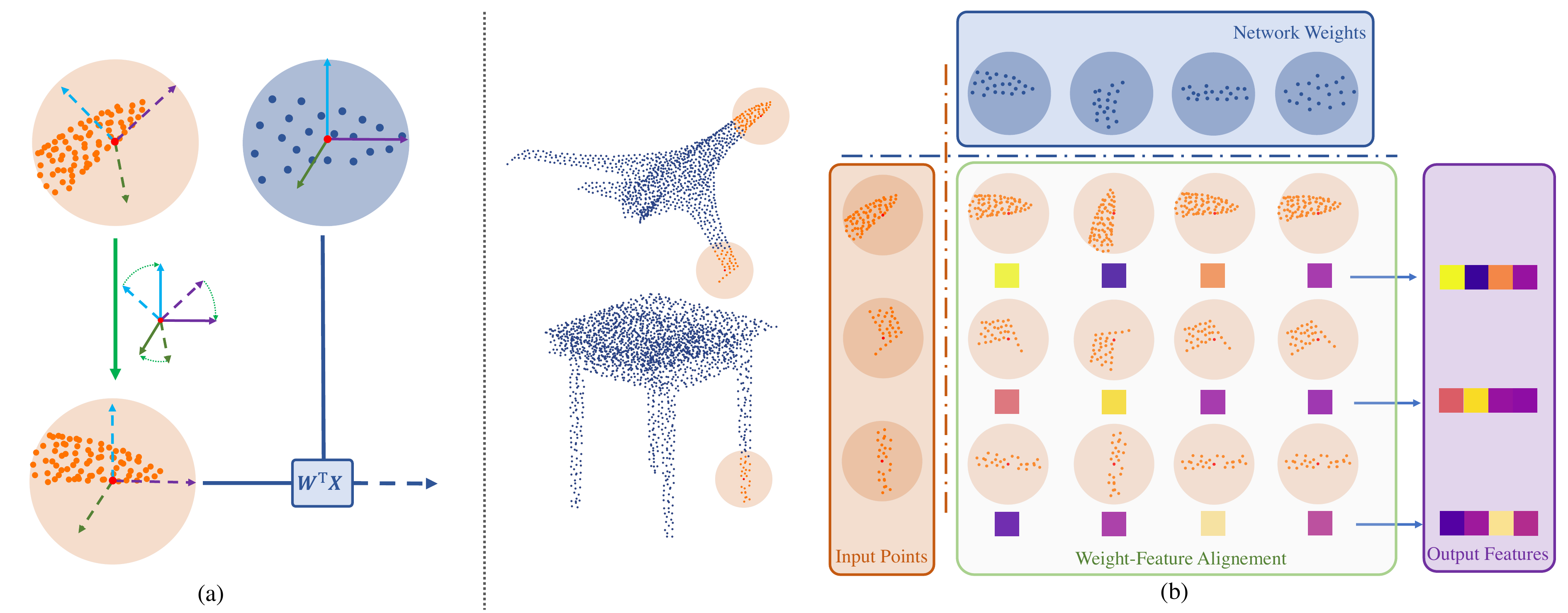}
    \centering
    \caption{\textbf{(a): Weight-Feature Alignment.} The local IRF (Invariant Reference Frame) is aligned with the principal axes of the network weights. \textbf{(b): The Motivation and Inspiration.} Feature extraction can be considered pattern matching. The network weights acting like the kernel of 2D CNN are responsible for certain patterns and will output active activation if they are matched with the patterns of the input points.}
    \label{fig:WFA}
\end{figure*}

\begin{figure}[t]
    \includegraphics[width=1.00\columnwidth]{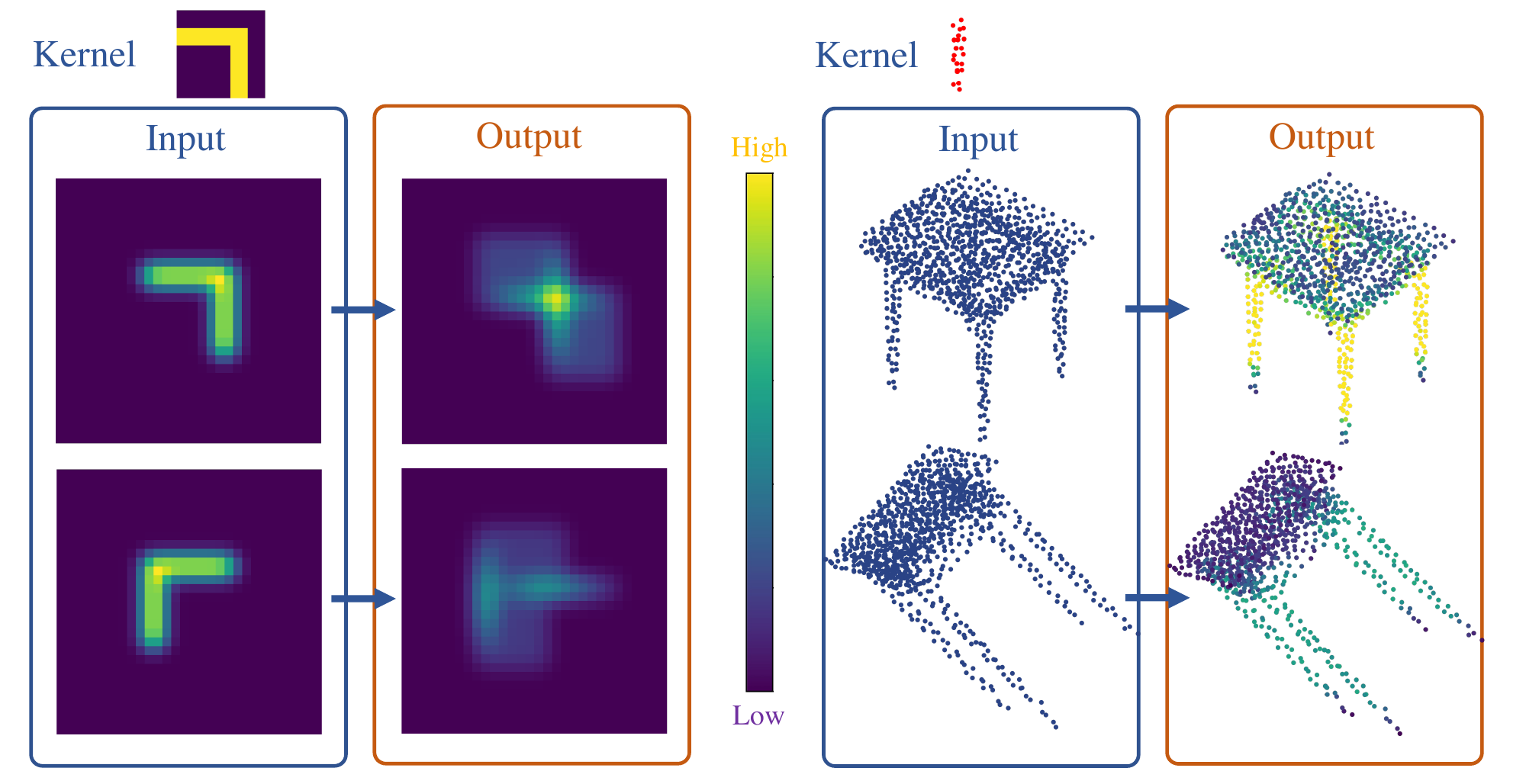}
    \centering
    \caption{\textbf{The perspective of Pattern Matching.} The extracted features can be considered responses about how much the patterns of the kernel and input are matched. For a kernel describing a certain pattern, the output returns an active response only when the orientation of the input pattern is as same as the kernel. Otherwise, the response will be low, which makes the network hard to recognize the patterns of orientations unseen in the training phase. }
    \label{fig:pattern_orientation_alignment}
\end{figure}

\section{Method}

\begin{figure*}[t]
    \includegraphics[width=0.96\textwidth]{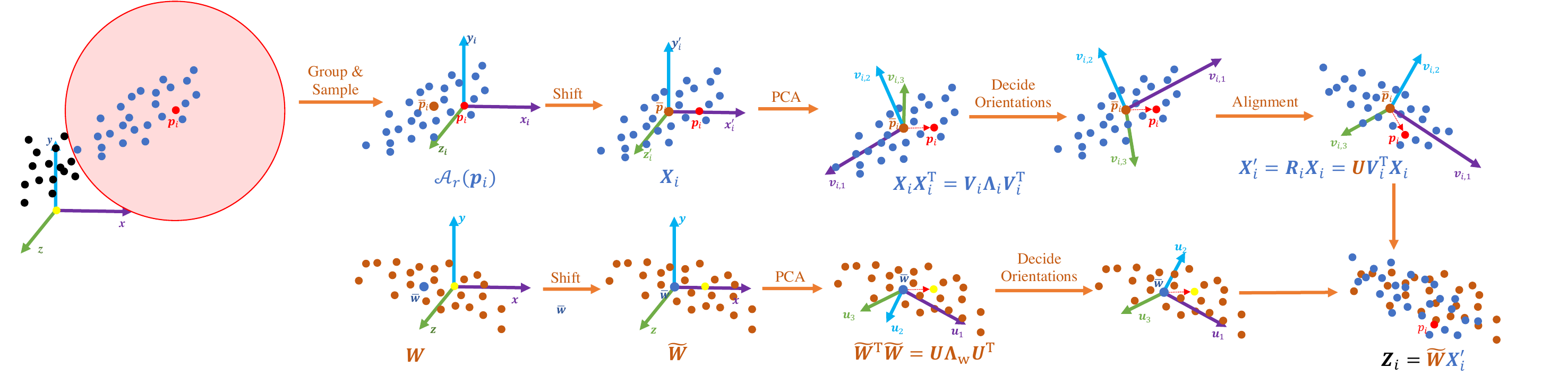}
    \centering
    \caption{\textbf{Weight-Feature Alignment.} (1) PCA are performed on the input points and network weights respectively to obtain the IRF (invariant reference frames) decided by $\boldsymbol{U}$ and $\boldsymbol{V}_{i}$. (2) The orientations of the IRF's axes are decided uniformly. Especially, the orientations of $\boldsymbol{V}_{i}$ and $\boldsymbol{U}$ are chosen so that the vectors $(\boldsymbol{p}_{i} - \bar{\boldsymbol{p}}_{i})$ and $-\bar{\boldsymbol{w}}$ locate at the first quadrant. (3) Points are aligned via the transformation $\boldsymbol{X}'_{i} = \boldsymbol{R}_{i}\boldsymbol{X}_{i}, \; \boldsymbol{R}_{i} = \boldsymbol{U}\boldsymbol{V}_{i}^{\mathrm{T}}$. }
    \label{fig:WFA_details}
\end{figure*}

\subsection{Preliminary \& Notations}

The point clouds can be represented as a disordered set of 3D points $\mathcal{P} = \{ \boldsymbol{p}_{i} \, \vert \, \boldsymbol{p}_{i} \in \mathbb{R}^{3 \times 1}, \; i = 1, 2, \dots, n\}$, and we use bold letter to presents the matrix format: $\boldsymbol{P} = {\left[ \boldsymbol{p}_{1}, \,  \boldsymbol{p}_{2}, \, \dots, \, \boldsymbol{p}_{n} \right]} \in \mathbb{R}^{3 \times c}$, where $\boldsymbol{p}_{i}$ is the column vector of $xyz$-Cartesian coordinates of $i$-th point. 
We use $\mathcal{A}_{r}(\boldsymbol{p}_{i})$ to denote the set of the point indices in $\boldsymbol{p}_{i}$'s neighbor area with radius of $r$, i.e., the neighbor points of $\boldsymbol{p}_{i}$ can be represented as: $\{ \boldsymbol{p}_{j} \, \vert \, j \in \mathcal{A}_{r}(\boldsymbol{p}_{i}) \}$. 
The barycenter of the neighboring points of point $\boldsymbol{p}_{i}$ is denoted as $\bar{\boldsymbol{p}}_{i} = \frac{1}{n_{r}}\sum_{j \in \mathcal{A}_{r}(\boldsymbol{p}_{i})} \boldsymbol{p}_{j}$.

\subsection{Motivation \& Inspiration.} 

\textbf{The Perspective of Feature Extraction.} The feature extraction in a local region can be abstracted as:

\begin{equation}
	\label{equ:activation_unexpected_abstracted}
	\boldsymbol{R} = f\left(\boldsymbol{L}(\mathcal{P}, \, \mathcal{O}), \; \boldsymbol{G}(\mathcal{P}, \, \mathcal{O})\right)
\end{equation}

\noindent where $\boldsymbol{R}$ represents the response, \ie the comprehensive representations of point clouds, $\boldsymbol{L}$ and $\boldsymbol{G}$ represents the patterns described by the points in the local and global regions, respectively. 
$\boldsymbol{L}$ and $\boldsymbol{G}$ are both the function of the spatial orientation of point clouds $\mathcal{O}$ and the \footnote{The intrinsic patterns refer to the patterns that are described by object's shape and remain unaffected by conformal geometric transformations, \eg rotation.}{\textit{intrinsic}} patterns $\mathcal{P}$. 
The feature extraction function $f$ integrates both local and global features to obtain the representation of the point clouds. 
Naturally, the representation is subject to spatial orientation, \ie rotations can affect the representations of point clouds describing the same object.

The target of rotation invariance learning is to eliminate the effect of the spatial orientation $\mathcal{O}$. 
So that the model can avoid fitting the patterns under all possible orientations via the data augmentation and eliminate the performance gap due to the orientations unseen in the training phase. 
And considering that the mainstream models extract the features of small scale in the primary layers and aggregate them from a larger scale in the subsequent layers, the feature in the deep layer naturally contains the global information. 
Therefore, in the primary layers of the model, we expect that the features are not affected by the global features $\boldsymbol{G}$. 
The final goal of rotation invariance learning can be written as:

\begin{equation}
	\label{equ:activation_expected_abstracted}
	\mathcal{R} = f\left(\boldsymbol{L}(\mathcal{P})\right),
\end{equation}

\noindent \ie the response $\boldsymbol{R}$ is the function only of local pattern $\boldsymbol{L}$. 

As discussed in the Section \textit{Related Works}, most recent IRF-based works \cite{xiao2020endowing,yu2020deep,spezialetti2020learning,li2021closer,fang2020rotpredictor} have a severe shortcoming: 
their IRFs rely on the global distribution of the point clouds, so their rotation-invariant features are inevitably affected by the backgrounds, unexpectedly. 
The features for the one certain pattern in these IRFs are still not unique. 
For the single-object point clouds or the point set describing small regions, the effect of this drawback is not very significant. 
But it severely limits the applicability on the open large-range scenes. 
Therefore, the solution to this problem is to \textbf{overcome the reliance on the information out of local region}.

\textbf{The Perspective of Pattern Recognition.} 
If we use $f_{\boldsymbol{F}}$ to represent the kernel function that detects the pattern $\boldsymbol{F}$, the response of the function $f$ can be written as:

\begin{equation}
	\label{equ:f_L_c}
	f_{\boldsymbol{F}}(\boldsymbol{L}) = \begin{cases}
		\boldsymbol{R}_{a}, & \text{if} \; \boldsymbol{L} \; \text{is matched with} \; \boldsymbol{F} \\
		\boldsymbol{R}_{i}, & \text{if} \; \boldsymbol{L} \; \text{is not matched with} \; \boldsymbol{F}
	\end{cases},
\end{equation}

\noindent where $\boldsymbol{R}_{a}$ and $\boldsymbol{R}_{i}$ represent the active and inactive response respectively.
However, when considering the orientation $\mathcal{O}$, as shown in Figure \ref{fig:pattern_orientation_alignment}, it will lead to the ambiguous logical inference, which can be abstracted as:

\begin{equation}
	\label{equ:unexpected_pattern_inference}
	\begin{aligned}
			\boldsymbol{R}_{a} \; & \Rightarrow \; ``\text{if} \; \boldsymbol{L} \; \text{is matched with} \; \boldsymbol{F}" \\
			\boldsymbol{R}_{i} \; & \Rightarrow \; ``\text{if} \; \boldsymbol{L} \; \text{is not matched with} \; \boldsymbol{F}" \; \mathrm{or} \; ``\text{if} \; \boldsymbol{L} \; \text{is not matched with} \; \boldsymbol{F}"
	\end{aligned}
\end{equation}

\noindent \ie we can not assert whether the input pattern $\boldsymbol{L}$ is matched to the pattern recognized by the kernel $f_{\boldsymbol{F}}$. 
We expected that the pattern matching degree could be exactly inferred from the activity of the response, \ie

\begin{equation}
	\label{equ:expected_pattern_inference}
	\begin{aligned}
			\boldsymbol{R}_{a} \; & \Rightarrow \; ``\text{if} \; \boldsymbol{L} \; \text{is matched with} \; \boldsymbol{F}" \\
			\boldsymbol{R}_{i} \; & \Rightarrow \; ``\text{if} \; \boldsymbol{L} \; \text{is not matched with} \; \boldsymbol{F}"
	\end{aligned}
\end{equation}

\noindent Therefore, we propose to \textbf{align the patterns of input and kernel} to realize the Equation (\ref{equ:expected_pattern_inference}).

Currently, there are two primary paradigms for extracting 3D point features. 
One is employing MLPs to implement continuous convolution or using Transformer\cite{vaswani2017attention} to directly process the raw 3D points \cite{qi2016pointnet,qi2017pointnetplusplus,wang2019dynamic,wu2018pointconv,thomas2019KPConv,li2018so,wang2018deep,li2018pointcnn,atzmon2018point,zhao2021point,engel2021point,park2022fast,guo2021pct}. 
The other is transforming the 3D points into voxels and then using CNNs or Vision Transformers for further processing \cite{graham2017submanifold,zhou2018voxelnet,mao2021voxel,he2022voxel,zhang2022pvt}.
Due to the voxelization process, which discretizes the 3D space into cubes, and the fact that 3D rotations occur in the continuous Lie-group space, theoretically, achieving invariance to rotations in three-dimensional space can be challenging.
Therefore, in this paper, we choose to directly process the raw 3D points using the classical PointNet++\cite{qi2017pointnetplusplus} architecture.

\subsection{Weight-Feature Alignment}

The weights of the linear layer in the neural network can be considered a set of points distributed in the same space as the input points:

\begin{equation}
	\boldsymbol{W} = \left[ \boldsymbol{w}_{1},\; \boldsymbol{w}_{2},\; \cdots,\; \boldsymbol{w}_{d} \right] \in \mathbb{R}^{c_{\mathrm{in}} \times d},
\end{equation}

\noindent where $\boldsymbol{w}_{i} \in \mathbb{R}^{c_{\mathrm{in}}}$ is the $i$-th row vector of the weights $\boldsymbol{W}$, $c_{\mathrm{in}}$ is the number of the channel of the input features, usually $c_{\mathrm{in}}$ is $3$. 
Supposed that the input points are denoted as $\boldsymbol{X} \in \mathbb{R}^{c_{\mathrm{in}} \times n}$, the feature extraction, \ie the matrix multiplication $\boldsymbol{Z} = \boldsymbol{W}^{\mathrm{T}}\boldsymbol{X}$ can be considered pattern matching. 

For a query point $\boldsymbol{p}_{i}$, we first perform PCA on its neighboring points to obtain the principal axes as the following: 

\begin{equation}
	\label{equ:pca_x}
	\boldsymbol{X}_{i}\boldsymbol{X}_{i}^{\mathrm{T}}\boldsymbol{V}_{i} = \boldsymbol{V}_{i}\boldsymbol{\Lambda}_{i}, \qquad i = 1, 2, 3,
\end{equation}

\noindent where $\boldsymbol{X}_{i} = \left[(\boldsymbol{p}_{j} - \bar{\boldsymbol{p}}_{i})\right]_{j \in \mathcal{A}_{r}(\boldsymbol{p}_{i})} \in \mathbb{R}^{3 \times n_{r}}$ is the relative coordinates of $n_{r}$ neighboring points, $\boldsymbol{V}_{i} = [ \boldsymbol{v}_{i,1},\; \boldsymbol{v}_{i,2},\; \boldsymbol{v}_{i,3} ]$ is the matrix composed of three unit principal components and $\boldsymbol{\Lambda}_{i} = \operatorname{diag}(\lambda_{i,1},\; \lambda_{i,2},\; \lambda_{i,3})$ is the diagonal matrix whose diagonal corresponds to the descent order of the eigenvalues $\lambda_{i,1} > \lambda_{i,2} > \lambda_{i,3}$.

Similarly, we perform PCA on the network weights:

\begin{equation}
	\label{equ:pca_w}
	\widetilde{\boldsymbol{W}}\widetilde{\boldsymbol{W}}^{\mathrm{T}}\boldsymbol{U} = \boldsymbol{U}\boldsymbol{\Lambda}_{\mathbf{w}},
\end{equation}

\noindent where $\widetilde{\boldsymbol{W}} = {\left[ \boldsymbol{w}_{1} - \bar{\boldsymbol{w}},\; \boldsymbol{w}_{2} - \bar{\boldsymbol{w}},\; \cdots,\; \boldsymbol{w}_{d} - \bar{\boldsymbol{w}} \right]} \in \mathbb{R}^{3 \times d}$, $\bar{\boldsymbol{w}} = \frac{1}{d}\sum_{k=1}^{d}\boldsymbol{w}_{k}$ is the barycenter of the learnable weights of the network, $\boldsymbol{U} = [ \boldsymbol{u}_{1},\; \boldsymbol{u}_{2},\; \boldsymbol{u}_{3}]$ and $\boldsymbol{\Lambda}_{\mathbf{w}} = \operatorname{diag}(\lambda_{w,1},\; \lambda_{w,2},\; \lambda_{w,3})$ are defined similarly with the PCA on query points.

However, there will be $8$ possibilities for the orientations of the principal axes, and this uncertainty will destroy the rotation invariance. Therefore we must make a rule to decide the orientations. 
As Figure \ref{fig:WFA_details} shows, the orientation of $\boldsymbol{V}_{i}$ is chosen to make the vector $(\boldsymbol{p}_{i} - \bar{\boldsymbol{p}}_{i})$ to be at the first quadrant. 
The orientation of $\boldsymbol{U}$ is chosen to make the vector $-\bar{\boldsymbol{w}}$ to be at the first quadrant.

The results of PCA on points and weights give two groups of new 3D basic vectors, respectively. 
Then, we align the axes of these two Cartesian reference frames, like the illustration in Figure \ref{fig:WFA}. 
Mathematically, the aligning operation is equivariant to the 3D rotation, which can be represented by:

\begin{equation}
	\label{equ:R_i}
	\boldsymbol{T}_{i} = \boldsymbol{U}\boldsymbol{V}_{i}^{\mathrm{T}}
\end{equation}

\noindent Then, the aligned points $\tilde{\boldsymbol{p}}_{j}, j \in \mathcal{A}_{r}(\boldsymbol{p}_{i})$ can be obtained by:

\begin{equation}
	\label{equ:p'}
	\scalemath{0.9}{
		\begin{aligned}
			\tilde{\boldsymbol{p}}_{j} & = \boldsymbol{T}_{i}\left(\boldsymbol{p}_{j} - \bar{\boldsymbol{p}}_{i}\right) \\ 
				& = \boldsymbol{v}_{i,1}^{\mathrm{T}}\left(\boldsymbol{p}_{j} - \bar{\boldsymbol{p}}_{i}\right)\boldsymbol{u}_{1} + \boldsymbol{v}_{i,2}^{\mathrm{T}}\left(\boldsymbol{p}_{j} - \bar{\boldsymbol{p}}_{i}\right)\boldsymbol{u}_{2} + \boldsymbol{v}_{i,3}^{\mathrm{T}}\left(\boldsymbol{p}_{j} - \bar{\boldsymbol{p}}_{i}\right)\boldsymbol{u}_{3},
		\end{aligned}
	}
\end{equation}

\noindent Then, we extract the features of the query point $\boldsymbol{p}_{i}$ by:

\begin{equation}
	\label{equ:mlp_model}
	\boldsymbol{y}_{i} = \widetilde{\boldsymbol{W}}^{\mathrm{T}}\boldsymbol{X}_{i}' + \boldsymbol{b}
\end{equation}

\noindent where $\boldsymbol{X}'_{i} = [ \tilde{\boldsymbol{p}}_{j_{1}}, \, \tilde{\boldsymbol{p}}_{j_{2}}, \, \dots, \, \tilde{\boldsymbol{p}}_{j_{n_{r}}} ] \in \mathbb{R}^{3 \times n_{r}}, \; j_{1}, \dots,j_{n_{r}} \in \mathcal{A}_{r}(\boldsymbol{p}_{i})$, $\boldsymbol{W}, \, \boldsymbol{b},$ are the learnable weights and bias. 
In the classical PointNet++-like architectures, we perform WFA in every sampling-and-grouping operation, \ie the set-abstraction module mentioned in \cite{qi2017pointnetplusplus}.

No matter how the points rotate, the local reference frame represented by the orthogonal matrix $\boldsymbol{T}_{i}$ will always be aligned with the principal axes of the network weights. 
For the rotation invariance proof of our WFA methods, please refer to the Appendix.

\subsection{Theoretical Results}

In our WFA method, the forward of the linear layers (ignoring the bias) can be represented as:

\begin{equation}
	\label{equ:first_wx}
	\boldsymbol{Z}_{i} = \widetilde{\boldsymbol{W}}^{\mathrm{T}}\boldsymbol{X}_{i}'
\end{equation}

\noindent where the symbols have the same meaning as the previous subsection. 
Our target is to construct a local IRF, \ie a linear transformation represented by an orthogonal matrix $\boldsymbol{T}_{i} \in \mathbb{R}^{3 \times 3}$, such that $\boldsymbol{X}'_{i} = \boldsymbol{T}_{i}\boldsymbol{X}_{i}$. 
As discussed above, to avoid the ambiguous inference described by Equation (\ref{equ:unexpected_pattern_inference}) that the response is not active when the patterns are matched, but the orientations are not aligned, we expect the network weights and input points are as aligned as possible. 
It can also be understood that we expect to minimize the \textit{distance} between the input points and weights. 
This goal is as same as the optimization target of the point clouds registration task.

\begin{theorem}
	\label{theorem:PCR}
	For each query point $\boldsymbol{p}_{i}$, $\boldsymbol{T}_{i}$ in Equation (\ref{equ:R_i}) is the solution of the following optimization problem:
	\begin{equation}
		\underset{\boldsymbol{T}}{\arg\min}\;\sum_{k=1}^{n}{\left\Vert \tilde{\boldsymbol{w}}_{\pi(k)} - \boldsymbol{T}\boldsymbol{x}_{k} \right\Vert}_{2}^{2},
	\end{equation}
	\noindent where $\tilde{\boldsymbol{w}}_{\pi(k)} = \boldsymbol{w}_{\pi(k)} - \bar{\boldsymbol{w}}$ and $\boldsymbol{x}_{j} = \boldsymbol{p}_{j} - \bar{\boldsymbol{p}}_{i}, \; j \in \mathcal{A}_{r}(\boldsymbol{p}_{i})$. The subscript $\pi(k)$ represents that $\tilde{\boldsymbol{w}}_{\pi(k)}$ is the closest (measured by Euclidian distance) point to the point $\boldsymbol{T}\boldsymbol{x}_{k}$, \ie,
	\begin{equation}
		\pi(k) = \underset{\tilde{\boldsymbol{w}}_{i} \in \boldsymbol{W}}{\arg\min}\; {\left\Vert \tilde{\boldsymbol{w}}_{i} -  \boldsymbol{T}\boldsymbol{x}_{k} \right\Vert}_{2}^{2}
	\end{equation}
	\noindent where $\tilde{\boldsymbol{w}}_{i} \in \boldsymbol{W}$ represents that $\tilde{\boldsymbol{w}}_{i}$ is one of $\boldsymbol{W}$'s column vectors.
\end{theorem}

The optimization problem in Theorem \ref{theorem:PCR} is as same as the point-to-point ICP algorithm \cite{besl1992method}. For detailed proof, please refer to the Appendix.

\begin{table*}[t]
	\centering
	\resizebox{\textwidth}{!}{
		\begin{tabular}{c|cccccccccccccccccc}
			\hline
			Methods   									& aero. & bag  & cap  & car  & chair & earph. & guitar & knife & lamp & laptop & motor & mug  & pistol & rocker & skate. & table & Class mIoU & Insta. mIoU \\ \hline
			PointNet\cite{qi2016pointnet}        		& 81.6  & 68.7 & 74.0 & 70.3 & 87.6  & 68.5   & 88.9   & 80.0  & 74.9 & 83.6   & 56.5  & 77.6 & 75.2   & 53.9   & 69.4   & 79.9  & 74.4       & -           \\
			PointNet++\cite{qi2017pointnetplusplus} 		& 79.5  & 71.6 & 87.7 & 70.7 & 88.8  & 64.9   & 88.8   & 78.1  & 79.2 & 94.9   & 54.3  & 92.0 & 76.4   & 50.3   & 68.4   & 81.0  & 76.7       & -           \\
			PointCNN\cite{li2018pointcnn}        		& 78.0  & 80.1 & 78.2 & 68.2 & 81.2  & 70.2   & 82.0   & 70.6  & 68.9 & 80.8   & 48.6  & 77.3 & 63.2   & 50.6   & 63.2   & 82.0  & 71.4       & -           \\
			DGCNN\cite{wang2019dynamic}           		& 77.7  & 71.8 & 77.7 & 55.2 & 87.3  & 68.7   & 88.7   & 85.5  & 81.8 & 81.3   & 36.2  & 86.0 & 77.3   & 51.6   & 65.3   & 80.2  & 73.3       & -           \\ 
			PRIN\cite{you2020pointwise}              	& 67.4  & 61.5 & 69.7 & 59.5 & 77.7  & 65.8   & 75.7   & 77.2  & 65.9 & 82.0   & 44.2  & 79.8 & 63.6   & 53.0   & 67.5   & 70.7  & 67.6       & -           \\
			RI-Conv\cite{zhang2019rotation}$\star$      & 80.6  & 80.2 & 70.7 & 68.8 & 86.8  & 70.4   & 87.2   & 84.3  & 78.0 & 80.1   & 57.3  & 91.2 & 71.3   & 52.1   & 66.6   & 78.5  & 75.3       & -           \\
			Li \etal\cite{li2021rotation}$\star$			& 81.4  & 84.5 & 85.1 & 75.0 & 88.2  & 72.4   & 90.7   & 84.4  & 80.3 & 84.0   & 68.8  & 92.6 & 76.1   & 52.1   & 74.1   & 80.0  & 79.4       & 82.5        \\
			LGR-Net\cite{zhao2022rotation}$\dagger$$\star$ & 81.7  & 78.1 & 82.5 & 75.1 & 87.6  & 74.5   & 89.4   & 86.1  & 83.0 & 86.4   & 65.3  & 92.6 & 75.2   & 64.1   & 79.8   & 80.5  & 80.1       & 82.8        \\
			Li \etal\cite{li2021closer} 					& 83.7  & 62.9 & 79.1 & 73.4 & 90.1  & 64.2   & 90.3   & 86.4  & 82.5 & 87.3   & 46.5  & 89.1 & 75.4   & 46.1   & 66.6   & 81.3  & 75.3       & 83.1        \\ 
			Zhang \etal\cite{zhang2020learning}			& -     & -    & -    & -    & -     & -      & -      & -     & -    & -      & -     & -    & -      & -      & -      & -     & 80.2       & -           \\ \hline
			Ours 										& 83.1 & 81.0 & 85.4 & 75.1 & 90.6  & 68.4   & 90.8   & 87.3  & 82.9 & 93.8   & 73.1  & 95.0 & 81.7   & 59.1   & 77.7   & 79.0  & \textbf{81.5} & \textbf{83.9} \\
			Ours$\dagger$ 								& 82.7  & 81.1 & 86.1 & 76.0 & 90.5  & 68.2   & 91.1   & 87.7  & 83.5 & 95.3   & 72.3  & 95.5 & 82.5   & 60.8   & 76.9   & 82.2  & \textbf{82.1} & \textbf{84.9} \\ \hline
		\end{tabular}
	}
	\caption{\textbf{Part Segmentation Results on ShapeNet Dataset.} The results in this table are all trained and evaluated with arbitrary rotation. The results of previous methods are referred from the their publications. The $\dagger$ symbol after the method name represents using the normal vectors and $\star$ represents using the constructed features.} 
	\label{tab:partseg_shapenet}
\end{table*}

\begin{table*}[h]
	\centering
	\resizebox{0.42\textwidth}{!}{
		\begin{tabular}{l|ccc}
			\hline
			Methods 									& z/z & AR/AR & z/AR \\ 
			\hline
			PointNet\cite{qi2016pointnet} 				& 88.5 & 70.5 & 16.3 \\
			PointNet++\cite{qi2017pointnetplusplus} 		& 89.3 & 85.0 & 28.6 \\
			DGCNN\cite{wang2019dynamic}					& 92.2 & 81.1 & 20.6 \\
			Spherical CNN\cite{esteves2018learning} 		& 88.9 & 86.9 & 76.9 \\
			$a^{2}$SCNN\cite{liu2018deep} 				& 89.6 & 88.7 & 87.9 \\ 
			SFCNN\cite{rao2019spherical} 				& 91.4 & 90.1 & 84.8 \\
			SFCNN\cite{rao2019spherical}$\dagger$ 		& 92.3 & 91.0 & 85.3 \\
			RotPredictor\cite{fang2020rotpredictor} 		& 92.1 & 90.8 & -    \\
			REQNN\cite{shen20203d} 						& 83.0 & 83.0 & 83.0 \\
			RI-Conv\cite{zhang2019rotation}$\star$ 		& 86.5 & 86.4 & 86.4 \\
			Triangle-Net\cite{xiao2021triangle}$\star$ 	& - 	   & 86.7 & -    \\
			SRI-Net\cite{sun2019srinet}$\star$			& 87.0 & 87.0 & 87.0 \\
			ClusterNet\cite{chen2019clusternet} 			& 87.1 & 87.1 & 87.1 \\
			SPH-Net\cite{poulenard2019effective} 		& 87.7 & 87.6 & 86.6 \\
			Yu \etal\cite{yu2020deep}$\dag$ 				& 89.2 & 89.2 & 89.2 \\
			Li \etal\cite{li2021rotation}$\star$ 		& 89.4 & 89.3 & 89.4 \\
			RI-GCN\cite{kim2020rotation} 				& 89.5 & 89.5 & 89.5 \\
			RI-GCN\cite{kim2020rotation}$\dag$ 			& 91.0 & 91.0 & 91.0 \\
			LGR-Net\cite{zhao2022rotation}$\dagger$$\star$ 	& 90.9 & 91.1 & 90.9 \\
			Zhang \etal\cite{zhang2020learning}			& 91.0 & 91.0 & 91.0 \\
			Li \etal\cite{li2021closer} 					& 91.6 & 91.6 & 91.6 \\ 
			\hline
			Ours(MSG) 									& \textbf{91.8} & \textbf{91.8} & \textbf{91.8} \\ 
			Ours(MSG)$\dagger$								& \textbf{92.1} & \textbf{92.1} & \textbf{92.1} \\ 
			\hline
		\end{tabular}
	}
	\caption{\textbf{Classification Results on ModelNet40 Dataset.} The meaning of the table head: 1). z/z: training and testing with azimuthal rotation; 2). AR/AR: training and testing with arbitrary rotation; 3). z/AR: training with azimuthal rotation and testing with arbitrary rotation. The results of previous methods are referred from the their publications. The $\dagger$ symbol after the method name represents using the normal vectors and $\star$ represents using the constructed features.}
	\label{tab:cls_modelnet40}
\end{table*}

\subsection{Advantages}

\textbf{Generalizability.} 
Compared to some previous methods \cite{xiao2020endowing,li2021closer,zhang2020learning,zhao2022rotation,li2021rotation} that rely on the global scene of the point clouds and do not give invariant features for certain patterns, our proposed WFA algorithm is an operator only influenced by local point distributions. 
The point coordinates in our constructed IRF are not only invariant to the rotation but also variant to the background scenes. 
Therefore, our method is more general and applicable to all scenes. 
It is a remarkable and significant improvement because the fundamental purpose of rotation invariance learning on point clouds is to improve the robustness in practical engineering where the point clouds usually describe open and complex scenes, such as the application of robot and auto-driving.

\textbf{interpretability.}
Compared to the previous approaches \cite{yu2020deep,spezialetti2020learning,fang2020rotpredictor,li2021closer,xiao2020endowing} that attempt to learn an orientation or combine the features of all possible orientations via an attention-like mechanism, our WFA is more interoperable. 
The local IRF constructed by WFA avoids the problem of ambiguous inference stated by Equation (\ref{equ:unexpected_pattern_inference}) and Figure \ref{fig:pattern_orientation_alignment}. 
When the features are aligned with the network weights, the response can reflect the matching degree between the input pattern and the kernel's pattern.

\section{Experiments}

\subsection{Implementation Details}

We implement our network via PyTorch \cite{PyTorch}. We will make the code public on GitHub upon the publication of this work. Without special instructions, we train all models from scratch. Our experiments are executed on a computer equipped with an Intel(R) Core(TM) i7-6950X CPU  and an NVIDIA TITAN RTX GPU. We employ the PointNet++ as the default architecture and the neighbor point sampling operation is as same as the one in PointNet++. We use the official PointNet++'s MSG setting. The other training details are explained in detail in the following sections.

\subsection{Results on the Single-Object Point Clouds}

We evaluate our method on ModelNet40 \cite{wu20153d} and ShapeNet \cite{chang2015shapenet} dataset. For the experiments on these two datasets, we employ the classical PointNet++ (MSG: Multiple Scale Grouping) \cite{qi2017pointnetplusplus} framework and use Adam optimizer with the initial learning rate of $0.001$. We train the model on ModelNet40 dataset for $200$ epochs with the learning rate decay rate of $0.7$, and the step size of $20$. We train the model on ShapeNet dataset for $250$ epochs with the learning rate decay rate of $0.5$, and the step size of $20$. For the experiments on these two datasets, we use the Cartesian coordinates and the normal vectors as the input point features. For the additional features of normal vectors, we project them into the IRF via Equation (\ref{equ:R_i}).

\subsubsection{Results on ModelNet40 Dataset.} 

ModelNet40 contains CAD models of 40 categories. Following the common practice of \cite{qi2016pointnet}, we use the pre-processed training set of $9843$ models and the testing set of $2468$ models. For all experiments, we use Furthest Point Sampling (FPS) to sample $1024$ points as input and evaluate by mean instance accuracy. The results are reported in Table \ref{tab:cls_modelnet40}. Our method outperforms all state-of-the-art rotation-robust methods (including equivariant and invariant works). Compared to the original PointNet++ trained and evaluated with arbitrary rotations, our method achieves $7.1\%$ improvement. 

\subsubsection{Results on ShapeNet Dataset.} 

ShapeNet dataset contains $16881$ models of $16$ categories, and the 3D points are annotated with $50$ part labels. Following the standard split, the train split contains $14007$ samples, and the testing set contains $2874$ samples. We show the per-category segmentation IoU (Intersection of Union) the category and instance mIoU (mean IoU) in the Table \ref{tab:partseg_shapenet}. The results show that our method outperforms the state-of-the-art methods. Compared to the baseline, PointNet++ trained and tested with arbitrary rotations, we can achieve $7.7\%$ improvements. Even for PointNet++(MSG) trained and evaluated without any rotation (Instance mIoU is 85.1 and Class mIoU is 81.9), the performance gap is very small, which demonstrates the robustness of our method.

\begin{table*}[th]
	\centering
		\begin{tabular}{c|c}
			\hline
			Methods                 				& mIoU. \\ \hline
			RI-Conv \cite{zhang2019rotation} 	& 22.0 	\\
			LGR-Net \cite{zhao2022rotation}	 	& 43.4 	\\ \hline
			Ours                    				& 44.6 	\\ \hline
		\end{tabular}
    \caption{\textbf{Semantic Segmentation Results on S3DIS Dataset.} The results are trained with Area $1-4$ and $6$, evaluated on Area $5$. }
	\label{tab:semseg_s3dis}
\end{table*}

\begin{table*}[th]
	\centering
		\begin{tabular}{c|ccc|ccc}
			\hline
			\multirow{2}{*}{Methods} & \multicolumn{3}{c|}{BEV AP} & \multicolumn{3}{c}{3D AP} \\ \cline{2-7} 
			                         & Easy   & Moderate   & Hard  & Easy  & Moderate  & Hard  \\ \hline
			Point RCNN               & 90.21  & 87.89      & 85.51 & 88.88 & 78.63     & 77.38 \\
			Point RCNN(WFA)          & 89.89  & 88.13      & 85.97 & 88.64 & 78.93     & 77.89 \\ \hline
		\end{tabular}
	\caption{\textbf{3D Object Detection Results on KITTI Dataset.} The results are evaluated on the \textit{Car} category.}
	\label{tab:3d_object_detection_kitti}
\end{table*}

\begin{table*}[th]
	\centering
		\begin{tabular}{c|ccc|c}
			\hline
			\multirow{2}{*}{Alignment}        & \multicolumn{3}{c|}{$\boldsymbol{U}$}                                    & \multirow{2}{*}{mean acc.} \\ \cline{2-4}
			                                  & $\boldsymbol{u}_{1}$ & $\boldsymbol{u}_{2}$ & $\boldsymbol{u}_{3}$ 		 &                       \\ \hline
			\multirow{6}{*}{$\boldsymbol{V}$} & $\boldsymbol{v}_{i,1}$ & $\boldsymbol{v}_{i,2}$ & $\boldsymbol{v}_{i,3}$ & 92.1                  \\
			                                  & $\boldsymbol{v}_{i,3}$ & $\boldsymbol{v}_{i,2}$ & $\boldsymbol{v}_{i,1}$ & 90.6                  \\
			                                  & $\boldsymbol{v}_{i,3}$ & $\boldsymbol{v}_{i,1}$ & $\boldsymbol{v}_{i,2}$ & 90.8                  \\
			                                  & $\boldsymbol{v}_{i,1}$ & $\boldsymbol{v}_{i,3}$ & $\boldsymbol{v}_{i,2}$ & 91.7                  \\
			                                  & $\boldsymbol{v}_{i,2}$ & $\boldsymbol{v}_{i,1}$ & $\boldsymbol{v}_{i,3}$ & 91.7                  \\
			                                  & $\boldsymbol{v}_{i,2}$ & $\boldsymbol{v}_{i,3}$ & $\boldsymbol{v}_{i,1}$ & 91.5                  \\ \hline
		\end{tabular}
	\caption{\textbf{The Ablation Study on ModelNet40 Dataset about Different Aligning Orders.} The notations in the table are defined as same as Equation (\ref{equ:R_i}). }
	\label{tab:ablation_alignment}
\end{table*}

\begin{table}[t]
	\centering
		\begin{tabular}{l|c|cc}
			\hline
			\multirow{2}{*}{Baseline Models} 			& ModelNet40 & \multicolumn{2}{c}{ShapeNet} \\ \cline{2-4} 
			                              				& mean acc.  & Class mIoU   & Insta. mIoU   \\ \hline
			PointNet \cite{qi2016pointnet}    			& 70.5/88.2  & 74.3/80.9    & 77.9/80.1            \\
			PointNet++ \cite{qi2017pointnetplusplus}    & 85.0/92.1  & 76.8/82.1    & 80.1/84.9        \\
			DGCNN \cite{wang2019dynamic}                & 81.1/92.0  & 73.3/81.7    & 77.6/79.7·              \\ \hline
		\end{tabular}
	\caption{\textbf{The Ablation Study about the Baseline Models.} The baseline results are based on our reimplementations. The models are trained and evaluated with arbitrary rotation. The results before and after "/" represent the baseline model and using our WFA. }
	\label{tab:ablation_baseline_models}
\end{table}

\subsection{Experiments on the Point Clouds of Open Scenes}

To prove the generalness of our method, besides the single-object scenes, we conduct experiments on the point clouds of complex open scenes. 

\subsubsection{Semantic Segmentation Results on the S3DIS Dataset.} 

S3DIS \cite{armeni20163d} semantic segmentation dataset has the point clouds scanned from $271$ rooms in $6$ Areas. The annotations contain $13$ semantic categories. The results in Table \ref{tab:semseg_s3dis} show our method can outperform the previous rotation-invariant methods (Note few previous works perform experiments on the point clouds of the complex scenes with multiple objects). 

\subsubsection{3D Object Detection Results on the KITTI Dataset.} 

To demonstrate the generalness of our method, we apply our WFA methods on the Point RCNN \cite{shi20193d} whose backbone is based on the PointNet++ and perform experiments on the 3D object detection task. We perform experiments on KITTI \cite{Geiger2012CVPR} dataset. KITTI 3D detection dataset contains $7481$ \textit{training} annotated samples. The annotated samples are divided into \textit{train} set with 3712 samples and \textit{val} set with 3769 samples. We follow the consistent setting with the original paper to train the model and evaluate on \textbf{Car} class. The results are shown in the Table \ref{tab:3d_object_detection_kitti}. The results show that applying our WFA operation obtains comparable performance and achieves improvements on the \textit{Moderate} and \textit{Hard} objects. Compared to the single-object or small-range point clouds, the point clouds captured in the outdoor open scenes are more complex: 1) the object patterns are influenced by the background scenes, 2) the points are much more sparse and the point densities vary dramatically. The results suggest the generalness of our WFA method.

\subsection{Ablation Studies} 

\subsubsection{Different Alignments.} If we only want to achieve invariance to rotations, we do not need to align the principal axes according to the descending order of the SVD eigenvalues. There are $6$ possible alignments. We perform an ablation study on the ModelNet40 dataset, and the results of different alignments are shown in Table \ref{tab:ablation_alignment}. It shows that the aligning along the order of eigenvalues achieves the best performance. The results in the first row and the second row deserves our attention. The second row represents the alignment opposite to our WFA, \ie aligning the max principal axis with the min one and aligning the min principal axis with the max one. Compared to the previous PCA-based works that learn the orientation of the principal axes through deep learning, our WFA gives an interpretable explanation that the different orientations of the principal axes are not equivariant to each other and that there is a better orientation.

\subsubsection{Different Baseline Models.}

Our WFA algorithm is a general operator that can be applied on arbitrary models that take raw points as input. In the Table \ref{tab:ablation_baseline_models}, we shows the results on $5$ popular baseline models, PointNet \cite{qi2016pointnet}, PointNet++ \cite{qi2017pointnetplusplus} and DGCNN \cite{wang2019dynamic}. The results suggest that our WFA can generally improve the rotation robustness of the mainstream frameworks.


\section{Conclusion}

In this paper, we propose a robust rotation-invariant algorithm, WFA (Weight-Feature Alignment). WFA proposes to transform the input points into the reference frame decided by the principal axes of the network weights. So the feature extraction can eliminate the effect of the pattern orientations and only concentrate on the invariant patterns. We perform the experiments on the point clouds of various scenes, including the single-object and open scenes point clouds. The experiment results demonstrate the robustness and the generalness of our method.

\bibliography{mybibfile}

\begin{thebibliography}{10}
\expandafter\ifx\csname url\endcsname\relax
  \def\url#1{\texttt{#1}}\fi
\expandafter\ifx\csname urlprefix\endcsname\relax\def\urlprefix{URL }\fi
\expandafter\ifx\csname href\endcsname\relax
  \def\href#1#2{#2} \def\path#1{#1}\fi

\bibitem{cohen2016group}
T.~Cohen, M.~Welling, Group equivariant convolutional networks, in:
  International conference on machine learning, PMLR, 2016, pp. 2990--2999.

\bibitem{cohen2018spherical}
T.~S. Cohen, M.~Geiger, J.~K{\"o}hler, M.~Welling, Spherical cnns, arXiv
  preprint arXiv:1801.10130.

\bibitem{weiler20183d}
M.~Weiler, M.~Geiger, M.~Welling, W.~Boomsma, T.~Cohen, 3d steerable cnns:
  Learning rotationally equivariant features in volumetric data, arXiv preprint
  arXiv:1807.02547.

\bibitem{cohen2016steerable}
T.~S. Cohen, M.~Welling, Steerable cnns, arXiv preprint arXiv:1612.08498.

\bibitem{shen20203d}
W.~Shen, B.~Zhang, S.~Huang, Z.~Wei, Q.~Zhang, 3d-rotation-equivariant
  quaternion neural networks, in: Computer Vision--ECCV 2020: 16th European
  Conference, Glasgow, UK, August 23--28, 2020, Proceedings, Part XX 16,
  Springer, 2020, pp. 531--547.

\bibitem{worrall2018cubenet}
D.~Worrall, G.~Brostow, Cubenet: Equivariance to 3d rotation and translation,
  in: Proceedings of the European Conference on Computer Vision (ECCV), 2018,
  pp. 567--584.

\bibitem{zhao2020quaternion}
Y.~Zhao, T.~Birdal, J.~E. Lenssen, E.~Menegatti, L.~Guibas, F.~Tombari,
  Quaternion equivariant capsule networks for 3d point clouds, in: European
  Conference on Computer Vision, Springer, 2020, pp. 1--19.

\bibitem{li2021leveraging}
X.~Li, Y.~Weng, L.~Yi, L.~J. Guibas, A.~Abbott, S.~Song, H.~Wang, Leveraging se
  (3) equivariance for self-supervised category-level object pose estimation
  from point clouds, Advances in Neural Information Processing Systems 34
  (2021) 15370--15381.

\bibitem{rao2019spherical}
Y.~Rao, J.~Lu, J.~Zhou, Spherical fractal convolutional neural networks for
  point cloud recognition, in: Proceedings of the IEEE/CVF Conference on
  Computer Vision and Pattern Recognition, 2019, pp. 452--460.

\bibitem{poulenard2019effective}
A.~Poulenard, M.-J. Rakotosaona, Y.~Ponty, M.~Ovsjanikov, Effective
  rotation-invariant point cnn with spherical harmonics kernels, in: 2019
  International Conference on 3D Vision (3DV), IEEE, 2019, pp. 47--56.

\bibitem{liu2018deep}
M.~Liu, F.~Yao, C.~Choi, A.~Sinha, K.~Ramani, Deep learning 3d shapes using
  alt-az anisotropic 2-sphere convolution, in: International Conference on
  Learning Representations, 2018.

\bibitem{yu2020deep}
R.~Yu, X.~Wei, F.~Tombari, J.~Sun, Deep positional and relational feature
  learning for rotation-invariant point cloud analysis, in: European Conference
  on Computer Vision, Springer, 2020, pp. 217--233.

\bibitem{zhang2019rotation}
Z.~Zhang, B.-S. Hua, D.~W. Rosen, S.-K. Yeung, Rotation invariant convolutions
  for 3d point clouds deep learning, in: 2019 International Conference on 3D
  Vision (3DV), IEEE, 2019, pp. 204--213.

\bibitem{zhang2022riconv++}
Z.~Zhang, B.-S. Hua, S.-K. Yeung, Riconv++: Effective rotation invariant
  convolutions for 3d point clouds deep learning, International Journal of
  Computer Vision 130~(5) (2022) 1228--1243.

\bibitem{deng2018ppf}
H.~Deng, T.~Birdal, S.~Ilic, Ppf-foldnet: Unsupervised learning of rotation
  invariant 3d local descriptors, in: Proceedings of the European Conference on
  Computer Vision (ECCV), 2018, pp. 602--618.

\bibitem{sun2019srinet}
X.~Sun, Z.~Lian, J.~Xiao, Srinet: Learning strictly rotation-invariant
  representations for point cloud classification and segmentation, in:
  Proceedings of the 27th ACM International Conference on Multimedia, 2019, pp.
  980--988.

\bibitem{li2021rotation}
X.~Li, R.~Li, G.~Chen, C.-W. Fu, D.~Cohen-Or, P.-A. Heng, A rotation-invariant
  framework for deep point cloud analysis, IEEE Transactions on Visualization
  and Computer Graphics.

\bibitem{xiao2021triangle}
C.~Xiao, J.~Wachs, Triangle-net: Towards robustness in point cloud learning,
  in: Proceedings of the IEEE/CVF Winter Conference on Applications of Computer
  Vision, 2021, pp. 826--835.

\bibitem{zhao2022rotation}
C.~Zhao, J.~Yang, X.~Xiong, A.~Zhu, Z.~Cao, X.~Li, Rotation invariant point
  cloud analysis: Where local geometry meets global topology, Pattern
  Recognition 127 (2022) 108626.

\bibitem{chen2019clusternet}
C.~Chen, G.~Li, R.~Xu, T.~Chen, M.~Wang, L.~Lin, Clusternet: Deep hierarchical
  cluster network with rigorously rotation-invariant representation for point
  cloud analysis, in: Proceedings of the IEEE/CVF Conference on Computer Vision
  and Pattern Recognition, 2019, pp. 4994--5002.

\bibitem{spezialetti2020learning}
R.~Spezialetti, F.~Stella, M.~Marcon, L.~Silva, S.~Salti, L.~di~Stefano,
  Learning to orient surfaces by self-supervised spherical cnns, in: NeurIPS,
  2020.

\bibitem{fang2020rotpredictor}
J.~Fang, D.~Zhou, X.~Song, S.~Jin, R.~Yang, L.~Zhang, Rotpredictor:
  Unsupervised canonical viewpoint learning for point cloud classification, in:
  2020 International Conference on 3D Vision (3DV), IEEE, 2020, pp. 987--996.

\bibitem{xiao2020endowing}
Z.~Xiao, H.~Lin, R.~Li, L.~Geng, H.~Chao, S.~Ding, Endowing deep 3d models with
  rotation invariance based on principal component analysis, in: 2020 IEEE
  International Conference on Multimedia and Expo (ICME), IEEE, 2020, pp. 1--6.

\bibitem{li2021closer}
F.~Li, K.~Fujiwara, F.~Okura, Y.~Matsushita, A closer look at
  rotation-invariant deep point cloud analysis, in: Proceedings of the IEEE/CVF
  International Conference on Computer Vision, 2021, pp. 16218--16227.

\bibitem{kim2020rotation}
S.~KIM, J.~Park, B.~Han, Rotation-invariant local-to-global representation
  learning for 3d point cloud, Advances in Neural Information Processing
  Systems 33.

\bibitem{zhang2020learning}
J.~Zhang, M.-Y. Yu, R.~Vasudevan, M.~Johnson-Roberson, Learning
  rotation-invariant representations of point clouds using aligned edge
  convolutional neural networks, in: 2020 International Conference on 3D Vision
  (3DV), IEEE, 2020, pp. 200--209.

\bibitem{qi2016pointnet}
C.~R. Qi, H.~Su, K.~Mo, L.~J. Guibas, Pointnet: Deep learning on point sets for
  3d classification and segmentation, arXiv preprint arXiv:1612.00593.

\bibitem{qi2017pointnetplusplus}
C.~R. Qi, L.~Yi, H.~Su, L.~J. Guibas, Pointnet++: Deep hierarchical feature
  learning on point sets in a metric space, arXiv preprint arXiv:1706.02413.

\bibitem{wang2019dynamic}
Y.~Wang, Y.~Sun, Z.~Liu, S.~E. Sarma, M.~M. Bronstein, J.~M. Solomon, Dynamic
  graph cnn for learning on point clouds, Acm Transactions On Graphics (tog)
  38~(5) (2019) 1--12.

\bibitem{wu2018pointconv}
W.~Wu, Z.~Qi, L.~Fuxin, Pointconv: Deep convolutional networks on 3d point
  clouds, arXiv preprint arXiv:1811.07246.

\bibitem{thomas2019KPConv}
H.~Thomas, C.~R. Qi, J.-E. Deschaud, B.~Marcotegui, F.~Goulette, L.~J. Guibas,
  Kpconv: Flexible and deformable convolution for point clouds, Proceedings of
  the IEEE International Conference on Computer Vision.

\bibitem{li2018so}
J.~Li, B.~M. Chen, G.~H. Lee, So-net: Self-organizing network for point cloud
  analysis, in: Proceedings of the IEEE conference on computer vision and
  pattern recognition, 2018, pp. 9397--9406.

\bibitem{wang2018deep}
S.~Wang, S.~Suo, W.-C. Ma, A.~Pokrovsky, R.~Urtasun, Deep parametric continuous
  convolutional neural networks, in: Proceedings of the IEEE Conference on
  Computer Vision and Pattern Recognition, 2018, pp. 2589--2597.

\bibitem{li2018pointcnn}
Y.~Li, R.~Bu, M.~Sun, W.~Wu, X.~Di, B.~Chen, Pointcnn: Convolution on
  x-transformed points, Advances in neural information processing systems 31
  (2018) 820--830.

\bibitem{atzmon2018point}
M.~Atzmon, H.~Maron, Y.~Lipman, Point convolutional neural networks by
  extension operators, ACM Transactions on Graphics (TOG) 37~(4) (2018) 1--12.

\bibitem{gybenko1989approximation}
G.~Gybenko, et~al., Approximation by superposition of sigmoidal functions,
  Mathematics of Control, Signals and Systems 2~(4) (1989) 303--314.

\bibitem{csaji2001approximation}
B.~C. Cs{\'a}ji, et~al., Approximation with artificial neural networks, Faculty
  of Sciences, Etvs Lornd University, Hungary 24~(48) (2001) 7.

\bibitem{esteves2018learning}
C.~Esteves, C.~Allen-Blanchette, A.~Makadia, K.~Daniilidis, Learning so (3)
  equivariant representations with spherical cnns, in: Proceedings of the
  European Conference on Computer Vision (ECCV), 2018, pp. 52--68.

\bibitem{vaswani2017attention}
A.~Vaswani, N.~Shazeer, N.~Parmar, J.~Uszkoreit, L.~Jones, A.~N. Gomez,
  {\L}.~Kaiser, I.~Polosukhin, Attention is all you need, Advances in neural
  information processing systems 30.

\bibitem{zhao2021point}
H.~Zhao, L.~Jiang, J.~Jia, P.~H. Torr, V.~Koltun, Point transformer, in:
  Proceedings of the IEEE/CVF international conference on computer vision,
  2021, pp. 16259--16268.

\bibitem{engel2021point}
N.~Engel, V.~Belagiannis, K.~Dietmayer, Point transformer, IEEE access 9 (2021)
  134826--134840.

\bibitem{park2022fast}
C.~Park, Y.~Jeong, M.~Cho, J.~Park, Fast point transformer, in: Proceedings of
  the IEEE/CVF Conference on Computer Vision and Pattern Recognition, 2022, pp.
  16949--16958.

\bibitem{guo2021pct}
M.-H. Guo, J.-X. Cai, Z.-N. Liu, T.-J. Mu, R.~R. Martin, S.-M. Hu, Pct: Point
  cloud transformer, Computational Visual Media 7 (2021) 187--199.

\bibitem{graham2017submanifold}
B.~Graham, L.~Van~der Maaten, Submanifold sparse convolutional networks, arXiv
  preprint arXiv:1706.01307.

\bibitem{zhou2018voxelnet}
Y.~Zhou, O.~Tuzel, Voxelnet: End-to-end learning for point cloud based 3d
  object detection, in: Proceedings of the IEEE conference on computer vision
  and pattern recognition, 2018, pp. 4490--4499.

\bibitem{mao2021voxel}
J.~Mao, Y.~Xue, M.~Niu, H.~Bai, J.~Feng, X.~Liang, H.~Xu, C.~Xu, Voxel
  transformer for 3d object detection, in: Proceedings of the IEEE/CVF
  International Conference on Computer Vision, 2021, pp. 3164--3173.

\bibitem{he2022voxel}
C.~He, R.~Li, S.~Li, L.~Zhang, Voxel set transformer: A set-to-set approach to
  3d object detection from point clouds, in: Proceedings of the IEEE/CVF
  Conference on Computer Vision and Pattern Recognition, 2022, pp. 8417--8427.

\bibitem{zhang2022pvt}
C.~Zhang, H.~Wan, X.~Shen, Z.~Wu, Pvt: Point-voxel transformer for point cloud
  learning, International Journal of Intelligent Systems 37~(12) (2022)
  11985--12008.

\bibitem{besl1992method}
P.~J. Besl, N.~D. McKay, Method for registration of 3-d shapes, in: Sensor
  fusion IV: control paradigms and data structures, Vol. 1611, Spie, 1992, pp.
  586--606.

\bibitem{you2020pointwise}
Y.~You, Y.~Lou, Q.~Liu, Y.-W. Tai, L.~Ma, C.~Lu, W.~Wang, Pointwise
  rotation-invariant network with adaptive sampling and 3d spherical voxel
  convolution, in: Proceedings of the AAAI Conference on Artificial
  Intelligence, Vol.~34, 2020, pp. 12717--12724.

\bibitem{PyTorch}
A.~Paszke, S.~Gross, F.~Massa, A.~Lerer, J.~Bradbury, G.~Chanan, T.~Killeen,
  Z.~Lin, N.~Gimelshein, L.~Antiga, A.~Desmaison, A.~Kopf, E.~Yang, Z.~DeVito,
  M.~Raison, A.~Tejani, S.~Chilamkurthy, B.~Steiner, L.~Fang, J.~Bai,
  S.~Chintala,
  \href{http://papers.neurips.cc/paper/9015-pytorch-an-imperative-style-high-performance-deep-learning-library.pdf}{Pytorch:
  An imperative style, high-performance deep learning library}, in: H.~Wallach,
  H.~Larochelle, A.~Beygelzimer, F.~d\textquotesingle Alch\'{e}-Buc, E.~Fox,
  R.~Garnett (Eds.), Advances in Neural Information Processing Systems 32,
  Curran Associates, Inc., 2019, pp. 8024--8035.
\newline\urlprefix\url{http://papers.neurips.cc/paper/9015-pytorch-an-imperative-style-high-performance-deep-learning-library.pdf}

\bibitem{wu20153d}
Z.~Wu, S.~Song, A.~Khosla, F.~Yu, L.~Zhang, X.~Tang, J.~Xiao, 3d shapenets: A
  deep representation for volumetric shapes, in: Proceedings of the IEEE
  conference on computer vision and pattern recognition, 2015, pp. 1912--1920.

\bibitem{chang2015shapenet}
A.~X. Chang, T.~Funkhouser, L.~Guibas, P.~Hanrahan, Q.~Huang, Z.~Li,
  S.~Savarese, M.~Savva, S.~Song, H.~Su, et~al., Shapenet: An information-rich
  3d model repository, arXiv preprint arXiv:1512.03012.

\bibitem{armeni20163d}
I.~Armeni, O.~Sener, A.~R. Zamir, H.~Jiang, I.~Brilakis, M.~Fischer,
  S.~Savarese, 3d semantic parsing of large-scale indoor spaces, in:
  Proceedings of the IEEE Conference on Computer Vision and Pattern
  Recognition, 2016, pp. 1534--1543.

\bibitem{shi20193d}
S.~Shi, X.~Wang, H.~P. Li, et~al., 3d object proposal generation and detection
  from point cloud, in: Proceedings of the IEEE conference on computer vision
  and pattern recognition, Long Beach, CA, USA, 2019, pp. 15--20.

\bibitem{Geiger2012CVPR}
A.~Geiger, P.~Lenz, R.~Urtasun, Are we ready for autonomous driving? the kitti
  vision benchmark suite, in: Conference on Computer Vision and Pattern
  Recognition (CVPR), 2012.

\end{thebibliography}

\newpage
\appendix
\onecolumn

\section{Appendix}

\subsection{The Rotation Invariance Proof of WFA}

\begin{theorem}
	\label{equ:rotation_invariance_WFA}
	For arbitrary rotation transformation $R$ applied on the point clouds $\mathcal{P} = \{\boldsymbol{p}_{i} \vert i = 1, \dots, n\}$, $\boldsymbol{y}_{i}$ in Equation (\ref{equ:mlp_model}) keeps invariant, \ie
	\begin{equation}
		R(\boldsymbol{y}_{i}) = \boldsymbol{y}_{i}
	\end{equation}
\end{theorem}

\begin{proof}
	We use the orthogonal matrix $\boldsymbol{T} \in \mathrm{SO}(3)$ represents the rotation $R$, and use $j_{1}, \dots \j_{n_{r}} \in \mathcal{A}_{r}(\boldsymbol{p}_{i})$ to represent the point indices of $n_{r}$ neighboring points. Then, we have,
	
	\begin{equation}
		\begin{aligned}
			R(\boldsymbol{y}_{i}) & = R(\widetilde{\boldsymbol{W}}^{\mathrm{T}}\boldsymbol{X}_{i}') = \widetilde{\boldsymbol{W}}^{\mathrm{T}}R(\boldsymbol{X}_{i}') \\
			& = \widetilde{\boldsymbol{W}}^{\mathrm{T}}R([\tilde{\boldsymbol{p}}_{j_{1}}, \dots, \tilde{\boldsymbol{p}}_{j_{n_{r}}}]) \\
			& = \widetilde{\boldsymbol{W}}^{\mathrm{T}}[R(\tilde{\boldsymbol{p}}_{j_{1}}), \dots, R(\tilde{\boldsymbol{p}}_{j_{n_{r}}})]) \\
		\end{aligned}
	\end{equation}

	\noindent Therefore, we only have to prove:

	\begin{equation}
		R(\tilde{\boldsymbol{p}}_{j}) = \tilde{\boldsymbol{p}}_{j}, \quad j \in \mathcal{A}_{r}(\boldsymbol{p}_{i})
	\end{equation}

	\begin{equation}
		\begin{aligned}
			R(\tilde{\boldsymbol{p}}_{j}) & = R(\boldsymbol{T}_{i}(\boldsymbol{p}_{j} - \bar{\boldsymbol{p}}_{i})) \\
			& = R(\boldsymbol{U}\boldsymbol{V}_{i}^{\mathrm{T}}(\boldsymbol{p}_{j} - \bar{\boldsymbol{p}}_{i})) \\
			& = \boldsymbol{U}{(\boldsymbol{T}\boldsymbol{V}_{i})}^{\mathrm{T}}(\boldsymbol{T}\boldsymbol{p}_{j} - \boldsymbol{T}\bar{\boldsymbol{p}}_{i}) \\
			& = \boldsymbol{U}\boldsymbol{V}_{i}^{\mathrm{T}}\boldsymbol{T}^{\mathrm{T}}\boldsymbol{T}(\boldsymbol{p}_{j} - \bar{\boldsymbol{p}}_{i}) \\
			& = \boldsymbol{U}\boldsymbol{V}_{i}^{\mathrm{T}}(\boldsymbol{p}_{j} - \bar{\boldsymbol{p}}_{i}) \\
			& = \tilde{\boldsymbol{p}}_{j}
		\end{aligned}
	\end{equation}

\end{proof}

\subsection{The Proof of Theorem \ref{theorem:PCR}}

\begin{proof}

	\begin{equation}
		\scalemath{0.8}{
			\begin{aligned}
				{\left\Vert \tilde{\boldsymbol{w}}_{k} - \boldsymbol{T}\boldsymbol{x}_{\pi(k)} \right\Vert}_{2}^{2} & = \left(\tilde{\boldsymbol{w}}_{k} - \boldsymbol{T}\boldsymbol{x}_{\pi(k)}\right)^{\mathrm{T}}\left(\tilde{\boldsymbol{w}}_{k} - \boldsymbol{T}\boldsymbol{x}_{\pi(k)}\right) \\[2ex]
				& = \tilde{\boldsymbol{w}}_{k}^{\mathrm{T}}\tilde{\boldsymbol{w}}_{k} + \boldsymbol{x}_{\pi(k)}^{\mathrm{T}}\boldsymbol{T}^{\mathrm{T}}\boldsymbol{T}\boldsymbol{x}_{\pi(k)} - \boldsymbol{x}_{\pi(k)}^{\mathrm{T}}\boldsymbol{T}^{\mathrm{T}}\tilde{\boldsymbol{w}}_{k} - \tilde{\boldsymbol{w}}_{k}^{\mathrm{T}}\boldsymbol{T}\boldsymbol{x}_{\pi(k)} \\[2ex]
				& = \Vert\tilde{\boldsymbol{w}}_{k}\Vert_{2}^{2} + \Vert\boldsymbol{x}_{\pi(k)}\Vert_{2}^{2} - 2\tilde{\boldsymbol{w}}_{k}^{\mathrm{T}}\boldsymbol{T}\boldsymbol{x}_{\pi(k)}
			\end{aligned}	
		}
	\end{equation}
	
	\noindent Because $\Vert\tilde{\boldsymbol{w}}_{k}\Vert_{2}^{2}$ and $\Vert\boldsymbol{x}_{\pi(k)}\Vert_{2}^{2}$ is invariant, the original optimization problem becomes to obtain the optimal solution $\boldsymbol{T}^{*}$ by:

	\begin{equation}
		\scalemath{0.86}{
			\begin{aligned}
				\boldsymbol{T}^{*} & = \underset{\boldsymbol{T}}{\arg\max} \left(\sum_{i=1}^{n} \tilde{\boldsymbol{w}}_{k}^{\mathrm{T}}\boldsymbol{T}\boldsymbol{x}_{\pi(k)}\right) \\
			 	& = \underset{\boldsymbol{T}}{\arg\max} \left(\mathrm{tr}\left(\widetilde{\boldsymbol{W}}^{\mathrm{T}}\boldsymbol{T}\boldsymbol{X}_{\pi}\right)\right)
			\end{aligned}
		}
	\end{equation}
	
	\noindent According to the property of the matrix trace $\mathrm{tr}(\boldsymbol{A}\boldsymbol{B}) = \mathrm{tr}(\boldsymbol{B}\boldsymbol{A})$ if the first dimension shape of $\boldsymbol{A}$ is as same as the second dimension shape of $\boldsymbol{B}$, we have:
	
	\begin{equation}
		\label{equ:max_optimization_problem}
		\scalemath{1.0}{
			\mathrm{tr}\left(\widetilde{\boldsymbol{W}}^{\mathrm{T}}\boldsymbol{T}\boldsymbol{X}_{\pi}\right) = \mathrm{tr}\left(\boldsymbol{T}\boldsymbol{X}_{\pi}\widetilde{\boldsymbol{W}}^{\mathrm{T}}\right)
		}
	\end{equation}
	
	\noindent Let $\boldsymbol{H} = \boldsymbol{X}_{\pi}\widetilde{\boldsymbol{W}}^{\mathrm{T}}$, and perform SVD on $\boldsymbol{H}$:
	
	\begin{equation}
		\boldsymbol{H} = \boldsymbol{U}_{\mathbf{H}}\boldsymbol{\Sigma}_{\mathbf{H}}\boldsymbol{V}^{\mathrm{T}}_{\mathbf{H}}
	\end{equation}
	
	\noindent where $\boldsymbol{U}_{\mathbf{H}}$ and $\boldsymbol{V}_{\mathbf{H}}$ are orthogonal matrixes, and $\boldsymbol{\Sigma}_{\mathbf{H}}$ are diagonal matrix whose diagonal entries are nonnegative singular values. Then, we have:
	
	\begin{equation}
		\label{equ:max_optimization_problem2}
		\scalemath{1.0}{
			\begin{aligned}
				\mathrm{tr}\left(\boldsymbol{T}\boldsymbol{X}_{\pi}\widetilde{\boldsymbol{W}}^{\mathrm{T}}\right) & = \mathrm{tr}\left(\boldsymbol{T}\boldsymbol{H}\right) \\
				& = \mathrm{tr}\left(\boldsymbol{T}\boldsymbol{U}_{\mathbf{H}}\boldsymbol{\Sigma}_{\mathbf{H}}\boldsymbol{V}_{\mathbf{H}}^{\mathrm{T}}\right) \\
				& = \mathrm{tr}\left(\boldsymbol{\Sigma}_{\mathbf{H}}\boldsymbol{V}_{\mathbf{H}}^{\mathrm{T}}\boldsymbol{T}\boldsymbol{U}_{\mathbf{H}}\right) \\
				& = \mathrm{tr}\left(\boldsymbol{\Sigma}_{\mathbf{H}}\boldsymbol{M}\right) \\
				& = \sigma_{1}m_{11} + \sigma_{2}m_{22} + \sigma_{3}m_{33}
			\end{aligned}
		}
	\end{equation}
	
	\noindent where $\boldsymbol{M} = \boldsymbol{V}_{\mathbf{H}}^{\mathrm{T}}\boldsymbol{T}\boldsymbol{U}_{\mathbf{H}}$ is an orthogonal matrix, $m_{11}, \; m_{22}, \; m_{33}$ are the diagonal elements of $\boldsymbol{M}$, and $\boldsymbol{\Sigma}_{\mathbf{H}} = \mathrm{diag}(\sigma_{1}, \; \sigma_{2}, \; \sigma_{3})$. Because the singular values are nonnegative and the absolute value of the entries of orthogonal matrix are less than $1$, Equation (\ref{equ:max_optimization_problem2}) reach the maximum value if and only if $\boldsymbol{M} = \boldsymbol{I}$, \ie $\boldsymbol{T}^{*} = \boldsymbol{V}_{\mathbf{H}}\boldsymbol{U}_{\mathbf{H}}^{\mathrm{T}}$.
	
	Perform PCA on $\boldsymbol{X}_{\pi}\boldsymbol{X}_{\pi}^{\mathrm{T}}$ and $\widetilde{\boldsymbol{W}}\widetilde{\boldsymbol{W}}^{\mathrm{T}}$:
	
	\begin{equation}
		\boldsymbol{X}_{\pi}\boldsymbol{X}_{\pi}^{\mathrm{T}} = \boldsymbol{V}_{\pi}\boldsymbol{\Lambda}_{\pi}\boldsymbol{V}_{\pi}^{\mathrm{T}}, \qquad \widetilde{\boldsymbol{W}}\widetilde{\boldsymbol{W}}^{\mathrm{T}} = \boldsymbol{U}_{\mathbf{w}}\boldsymbol{\Lambda}_{\mathbf{w}}\boldsymbol{U}_{\mathbf{w}}^{\mathrm{T}},
	\end{equation}
	
	\noindent where $\boldsymbol{U}_{\pi}, \; \boldsymbol{U}_{\mathbf{w}}$ are the orthogonal matrixes whose column vectors are the eigenvectors, and $\boldsymbol{\Lambda}_{\mathbf{x}}, \; \boldsymbol{\Lambda}_{\mathbf{w}}$ are the diagonal eigenvalue matrixes.
	
	Perform SVD on $\boldsymbol{X}_{\pi}$ and $\tilde{\boldsymbol{W}}$, 
	
	\begin{equation}
		\boldsymbol{X}_{\pi} = \boldsymbol{U}_{\pi}\boldsymbol{\Sigma}'_{\pi}\boldsymbol{V}_{\pi}^{\mathrm{T}}, \qquad \widetilde{\boldsymbol{W}} = \boldsymbol{U}_{\mathbf{w}}\boldsymbol{\Sigma}'_{\mathbf{w}}\boldsymbol{V}_{\mathbf{w}}^{\mathrm{T}}
	\end{equation}
	
	\noindent where $\boldsymbol{U}_{\mathbf{w}}$ is defined the same as $\boldsymbol{U}$ in Equation (\ref{equ:R_i}), and $\boldsymbol{V}_{\pi}$ can be considered equaled with $\boldsymbol{V}_{i}$ when $\boldsymbol{X}_{\pi}$ and $\boldsymbol{X}_{i}$ in Equation (\ref{equ:pca_x}) obey the independent Identically distribution and the amount of the neighboring points is large enough (Because the covariance matrix is identical for a same distribution.). Then, we have 
	
	\begin{equation}
		\begin{aligned}
			\boldsymbol{X}_{\pi}\boldsymbol{X}_{\pi}^{\mathrm{T}} & = \boldsymbol{U}_{\pi}\boldsymbol{\Sigma}'_{\pi}\boldsymbol{V}_{\pi}^{\mathrm{T}}\boldsymbol{V}_{\pi}\boldsymbol{\Sigma}'_{\pi}\boldsymbol{U}_{\pi}^{\mathrm{T}} = \boldsymbol{U}_{\pi}\boldsymbol{\Sigma}_{\pi}\boldsymbol{U}_{\pi}^{\mathrm{T}} \\[2ex]
			\widetilde{\boldsymbol{W}}\widetilde{\boldsymbol{W}}^{\mathrm{T}} & = \boldsymbol{U}_{\mathbf{w}}\boldsymbol{\Sigma}'_{\mathbf{w}}\boldsymbol{V}_{\mathbf{w}}\boldsymbol{V}_{\mathbf{w}}\boldsymbol{\Sigma}'_{\mathbf{w}}\boldsymbol{U}_{\mathbf{w}}^{\mathrm{T}} = \boldsymbol{U}_{\mathbf{w}}\boldsymbol{\Sigma}_{\mathbf{w}}\boldsymbol{U}_{\mathbf{w}}^{\mathrm{T}}
		\end{aligned}
	\end{equation}
	
	\noindent Therefore, $\boldsymbol{H}\boldsymbol{H}^{\mathrm{T}}$ can be diagonalized as:
	
	\begin{equation}
		\begin{aligned}
			\boldsymbol{H}\boldsymbol{H}^{\mathrm{T}} & = \boldsymbol{X}_{\pi}\widetilde{\boldsymbol{W}}^{\mathrm{T}}\widetilde{\boldsymbol{W}}\boldsymbol{X}_{\pi}^{\mathrm{T}} \\
			& = \boldsymbol{U}_{\pi}\boldsymbol{\Sigma}'_{\pi}\boldsymbol{V}_{\pi}^{\mathrm{T}} \boldsymbol{V}_{\mathbf{w}}\boldsymbol{\Sigma}'_{\mathbf{w}}\boldsymbol{U}_{\mathbf{w}}^{\mathrm{T}}\boldsymbol{U}_{\mathbf{w}}\boldsymbol{\Sigma}'_{\mathbf{w}}\boldsymbol{V}_{\mathbf{w}}^{\mathrm{T}}\boldsymbol{V}_{\pi}\boldsymbol{\Sigma}'_{\pi}\boldsymbol{U}_{\pi}^{\mathrm{T}} \\
			& = \boldsymbol{U}_{\pi}\boldsymbol{\Sigma}'_{\mathbf{H}}\boldsymbol{U}_{\pi}^{\mathrm{T}} \\
		\end{aligned}
	\end{equation}
	
	\noindent where $\boldsymbol{\Sigma}'_{\mathbf{H}} = \boldsymbol{\Sigma}'_{\pi}\boldsymbol{V}_{\pi}^{\mathrm{T}} \boldsymbol{V}_{\mathbf{w}}\boldsymbol{\Sigma}'_{\mathbf{w}}\boldsymbol{U}_{\mathbf{w}}^{\mathrm{T}}\boldsymbol{U}_{\mathbf{w}}\boldsymbol{\Sigma}'_{\mathbf{w}}\boldsymbol{V}_{\mathbf{w}}^{\mathrm{T}}\boldsymbol{V}_{\pi}\boldsymbol{\Sigma}'_{\pi}$ is a diagonal matrix.	Meanwhile,
	
	\begin{equation}
		\begin{aligned}
			\boldsymbol{H}\boldsymbol{H}^{\mathrm{T}} & = \boldsymbol{U}_{\mathbf{H}}\boldsymbol{\Sigma}_{\mathbf{H}}\boldsymbol{V}^{\mathrm{T}}_{\mathbf{H}}\boldsymbol{V}_{\mathbf{H}}\boldsymbol{\Sigma}_{\mathbf{H}}\boldsymbol{U}_{\mathbf{H}}^{\mathrm{T}} = \boldsymbol{U}_{\mathbf{H}}\boldsymbol{\Sigma}''_{\mathbf{H}}\boldsymbol{U}_{\mathbf{H}}^{\mathrm{T}} \\
		\end{aligned}
	\end{equation}
	
	\noindent where $\boldsymbol{\Sigma}''_{\mathbf{H}} = \boldsymbol{\Sigma}_{\mathbf{H}}\boldsymbol{V}^{\mathrm{T}}_{\mathbf{H}}\boldsymbol{V}_{\mathbf{H}}\boldsymbol{\Sigma}_{\mathbf{H}}$ is a diagonal matrix. According to the property of matrix diagonalization, we have 
	
	\begin{equation}
		\boldsymbol{U}_{\mathbf{H}} = \boldsymbol{V}_{\pi},
	\end{equation}
	
	Similarly, we can obtain $\boldsymbol{V}_{\mathbf{H}} = \boldsymbol{U}_{\mathbf{w}}$ via diagonalizing $\boldsymbol{H}^{\mathrm{T}}\boldsymbol{H}$. Therefore, 
	
	\begin{equation}
		\boldsymbol{T}^{*} = \boldsymbol{U}_{\mathrm{w}}\boldsymbol{V}_{\pi}^{\mathrm{T}}
	\end{equation}

\end{proof}

\end{document}